\documentclass[a4paper,11pt]{article}

\usepackage[plain,noend]{algorithm2e}
\usepackage{tikz}	
\usepackage{amsmath}	
\usepackage{amsthm}	
\usepackage[numbers]{natbib}	
\usepackage[margin=1in]{geometry}

\newtheorem{innercustomgeneric}{\customgenericname}
\providecommand{\customgenericname}{}

\newtheorem{assumption}{Assumption}
\newtheorem{definition}{Definition}
\newtheorem{theorem}{Theorem}
\newtheorem{lemma}{Lemma}

\makeatletter
\def\blfootnote{\gdef\@thefnmark{}\@footnotetext}
\makeatother

\makeatother
\newcommand\independent{\protect\mathpalette{\protect\independenT}{\perp}}
\def\independenT#1#2{\mathrel{\rlap{$#1#2$}\mkern2mu{#1#2}}}

\title{The Reduced PC-Algorithm: Improved Causal Structure Learning in Large Random Networks}

\author{
        Arjun Sondhi \\
        Department of Biostatistics\\
        University of Washington\\
        asondhi@uw.edu \\
            \and
        Ali Shojaie \\
        Department of Biostatistics\\
        University of Washington\\
        ashojaie@uw.edu
}

\begin{document}

\maketitle

\begin{abstract}
We consider the task of estimating a high-dimensional directed acyclic graph, given observations from a linear structural equation model with arbitrary noise distribution. By exploiting properties of common random graphs, we develop a new algorithm that requires conditioning only on small sets of variables. The proposed algorithm, which is essentially a modified version of the PC-Algorithm, offers significant gains in both computational complexity and estimation accuracy. In particular, it results in more efficient and accurate estimation in large networks containing hub nodes, which are common in biological systems. We prove the consistency of the proposed algorithm, and show that it also requires a less stringent faithfulness assumption than the PC-Algorithm. Simulations in low and high-dimensional settings are used to illustrate these findings. An application to gene expression data suggests that the proposed algorithm can identify a greater number of clinically relevant genes than current methods.
\end{abstract}

\section{Introduction}

Directed acyclic graphs, or DAGs, are commonly used to represent causal relationships in complex biological systems. For example, in gene regulatory networks, directed edges represent regulatory interactions among genes, which are  represented as nodes of the graph. While causal effects in biological networks can be accurately inferred from perturbation experiments \citep{shojaie2013ripe}---including single or double gene knockouts \citep{ko1, ko2}---these are costly to run. Estimating DAGs from observational data is thus an important exploratory task for generating causal hypotheses \citep{friedman, jansen}, and designing more efficient experiments. 

Since the number of possible directed graphs grows super-exponentially in the number of nodes, estimation of DAGs is an NP-hard problem \citep{chick}. Methods of estimating DAGs from observational data can be broadly categorized into three classes. The first class, score-based methods, search over the space of all possible graphs, and attempt to maximize a goodness-of-fit score, generally using a greedy algorithm. Examples include the hill climbing and tabu search algorithms \citep{bnlearn}, as well as Bayesian approaches \citep{eaton2012bayesian}. The second class, constraint-based methods, first estimate the graph skeleton by performing conditional independence tests; the skeleton of a directed acyclic graph is the undirected graph obtained by removing the direction of edges. Information from conditional independence relations is then used to partially orient the edges of the graph. The resulting completed partially directed acyclic graph (CPDAG) represents the class of all directed acyclic graphs that are Markov equivalent, and therefore not distinguishable from observational data. The most well-known constraint-based method is the PC-Algorithm \citep{sprites}, which was popularized by \citet{kalisch} for high-dimensional estimation. Finally, hybrid methods combine score and constraint-based approaches. For example, the Max-Min Hill Climbing algorithm \citep{mmhc} estimates the skeleton using a constraint-based method, and then orients the edges by using a greedy search algorithm. Sparsity-inducing regularization approaches have also been used to develop efficient hybrid methods \citep{schmidt}.

Estimating DAGs in high dimensions introduces new computational and statistical challenges. Until recently, graph recovery in high dimensions was only established for the PC-Algorithm \citep{kalisch} and hybrid constraint-based methods \citep{ha2016penpc}. While the recent work of \citet{nandy2015} extends these results to score-based algorithms and their hybrid extensions, the PC-Algorithm is still considered a gold standard in high-dimensional sparse settings, due to its polynomial time complexity \citep{kalisch}. Moreover, constraint-based methods are indeed the building blocks of various hybrid approaches. Therefore, we primarily focus on constraint-based methods in this paper.

Despite its appealing features, the PC-Algorithm entails several properties that do not scale well to common high-dimensional settings. Specifically, large real-world biological systems are known to commonly be sparse graphs containing a small number of highly connected \emph{hub} nodes \citep{nets2, nets1}. In such graphs, the average node degree will be small, while the maximum tends to be much larger, and increases with the number of nodes. This is particularly problematic for the PC-Algorithm, whose computational and sample complexities scale with the \emph{maximum node degree} in the graph. Moreover, the recent work by \citet{uhler} shows that the distributional assumptions required for high-dimensional consistency of the PC-Algorithm are overly restrictive in practice, and that the class of graphs which do not satisfy these assumptions is large. Although work has been done by \citet{peterscausalnoise} on estimating DAGs defined over a larger class of probability models, the resulting methods also do not scale to high dimensions.

A common limitation of existing methods for estimating DAGs is that they do not account for structural properties of large networks. For instance, the PC-Algorithm only incorporates the sparsity of the network, by assuming that the maximum node degree in the graph skeleton is small relative to the sample size. However, real-world networks, particularly those observed in biology, are known to posses a number of other important properties. Of particular interest in estimating DAGs is the so-called \emph{local separation property} of large networks \citep{anandkumar}, which implies that the number of short paths between any two nodes is bounded. This property is observed in many large sparse networks, including polytrees, Erd\H{o}s-R{\'e}nyi, power law, and small world graphs \citep{durrett}. Power law graphs are of particular interest in many biological applications, as they allow for the presence of hub nodes. 

In this paper, we propose a low-complexity constraint-based method for estimating high-dimensional sparse DAGs. The new method, termed \emph{reduced PC} (rPC), exploits the local separation property of large random networks, which was used by \citet{anandkumar} in estimation of undirected graphical models. 
We show that rPC can consistently estimate the skeleton of high-dimensional DAGs by conditioning only on sets of small cardinality. This is in contrast to previous heuristic DAG learning approaches that set an upper bound on the number of parents for each node \citep{AC94, HD93}, which is an assumption that cannot be justified in many real-world networks. 
We also show that computational and sample complexities of rPC only depend on average sparsity of the graph---a notion that is made more precise in Sections~\ref{sec:rpct} and~\ref{sec:asymp}. This leads to considerable advantages over the PC-Algorithm, whose computational and sample complexities scale with the maximal node degree. Moreover, these properties hold for linear structural equation models \citep{shojaie} with arbitrary noise distributions, and require a weaker faithfulness conditions on the underlying probability distributions than the PC-Algorithm. 
We present two version of the rPC algorithm: a ``full" version for which we provide theoretical guarantees, and an approximate version which is much faster and performs almost identically in practice.

The rest of the paper is organized as follows. In Section~\ref{sec:prelim} we review basic properties of graphical models over DAGs, and give a short overview of the PC-Algorithm. Our new algorithm is presented in Section~\ref{sec:rpct} and its properties, including consistency in high dimensions are discussed in Section~\ref{sec:asymp}. Results of simulation studies and a real data example concerning the estimation of gene regulatory networks are presented in Sections~\ref{sec:sims} and \ref{sec:appl}, respectively. We end with a brief discussion in Section~\ref{sec:disc}. Technical proofs and additional simulations are presented in the appendices. 

\section{Preliminaries}\label{sec:prelim}

In this section, we review relevant properties of graphical models defined over DAGs, and briefly describe the theory and implementation of the PC-Algorithm.

\subsection{Background}

For $p$ random variables $X_1, \dots, X_p$, we define a graph $G = (V, E)$ with vertices, or nodes, $V = \left\{1, \dots, p \right\}$ such that variable $X_j$ corresponds to node $j$. The edge set $E \subset V \times V$ contains directed edges; that is, $(j,k) \in E$ implies $(k,j) \not \in E$. Furthermore, there are no directed cycles in $G$. We denote an edge from $j$ to $k$ as $j \rightarrow k$ and call $j$ a parent of $k$, and $k$ a child of $j$. The set of parents of node $k$ is denoted $pa(k)$, while the set of nodes adjacent to it, or all of $k$'s parents and children, is denoted $adj(k)$. These notations are also used for the corresponding random variable $X_k$. We assume there are no hidden common parents of node pairs (that is, no unmeasured confounders). The degree of node $k$ is defined as the number of nodes which are adjacent to it, $|adj(k)|$; we denote the maximal degree in the graph as $d_{max}$. A triplet of nodes $(i, j,k)$ is called an \emph{unshielded triple} if $i$ and $j$ are adjacent to $k$ but $i$ and $j$ are not adjacent. An unshielded triple is called a \emph{v-structure} if $i \rightarrow j \leftarrow k$. 

We assume random variables follow a linear structural equation model (SEM), 
\begin{equation}
X_k = \sum_{j \in pa(k)} \rho_{jk} X_j + \epsilon_k,
\label{sem}
\end{equation}
where for $k = 1, \ldots, p$, $\epsilon_k$ are independent random variables with finite variance, and $\rho_{jk}$ are fixed unknown constants. The directed Markov property, stated below, is usually assumed in order to connect the joint probability distribution of $X_1, \ldots, X_p$ to the structure of the graph $G$.

\begin{definition}
A probability distribution is Markov on a DAG $G = (V,E)$ if, conditional on its parents, every random variable $X_k$ is independent of its non-descendants; that is, $X_k \independent X_j \mid pa(X_k)$ for all $j \in V$ which are non-descendants of $k$. 
\end{definition}
Although this assumption allows us to connect conditional independence relationships to the DAG structure, there are generally multiple graphs that generate the same distribution under the Markov property. More concretely, DAGs are Markov equivalent if they have the same skeleton and the same set of v-structures. Therefore, constraint-based methods focus primarily on estimating the skeleton of the DAGs from observational data. Conditional independence relations identified when learning the skeleton are then used to orient some of the edges to obtain the CPDAG, which represents the Markov equivalence class of directed graphs \citep{kalisch}. 

We next define d-separation, a graphical property which is used to read conditional independence relationships from the DAG structure.
\begin{definition}
In a DAG $G$, two nodes $k_1$ and $k_2$ are d-separated by a set $S$ if and only if, for all paths $\pi$ between $k_1$ and $k_2$: \\
(i) $\pi$ contains a chain $i \rightarrow m \rightarrow j$ or a fork $i \leftarrow m \rightarrow j$ such that the middle node $m$ is in $S$, or \\
(ii) $\pi$ contains an inverted fork (or collider) $i \rightarrow m \leftarrow j$ such that the middle node $m$ is not in $S$ and no descendant of $m$ is in $S$. 
\end{definition}
A path $\pi$ which does not satisfy the requirements of this definition is known as a \textit{d-connecting path}. Using observed data, d-separations in a graph $G$ can be identified based on conditional independence relationships. To this end, we require the following assumption, known as \textit{faithfulness}, on the probability distribution of random variable on $G$. 
\begin{definition}
A probability distribution is faithful to a DAG $G$ if $X_i \independent X_j \mid X_S$ whenever $i$ and $j$ are d-separated by $S$. 
\end{definition}

\subsection{The PC-Algorithm}\label{sec:pc}

Together, d-separation and faithfulness suggest a simple algorithm for recovering the DAG skeleton. If we discover that $X_i \independent X_j \mid S$ for some set $S$, then there cannot be an edge $(i,j) \in E$. Conversely, if we discover $X_i \not \independent X_j \mid S$ for all possible sets $S$, then there must be an edge $(i,j) \in E$. Therefore, under faithfulness, an obvious strategy for skeleton estimation would be to test all possible conditional independence relations for each pair of variables; that is, test whether $X_i \independent X_j \mid S$ for any $S \subset V \setminus \left\{ i,j \right\}$. While this strategy is computationally infeasible for large $p$, and statistically problematic when $p > n$, it forms the basis of the PC-Algorithm. The PC-Algorithm starts with a complete undirected graph and deletes edges $(i,j)$ if a set $S$ can be found such that $X_i \independent X_j \mid S$. The algorithm also uses the fact that if such an $S$ exists, then there exists a set $S'$ such that all nodes in $S'$ are adjacent to $i$ or $j$ and $X_i \independent X_j \mid S'$. Thus, at each step of the algorithm, only local neighbourhoods need to be examined in order to find the separating sets. 

Although consistent for sufficiently sparse high-dimensional DAGs, the PC-Algorithm's computational and sample complexity scale with the maximal degree of the graph, $d_{max}$. Specifically, the algorithm's computational complexity is $O\big(p^{d_{max}}\big)$ and its sample complexity is $\Omega\big\{ \max(\log p, d_{max}^{1/b}) \big\}$ for some $b \in (0,1]$. This is problematic for graphs with highly connected hub nodes, which are common in real-world networks \citep{nets2, nets1}. In such graphs, $d_{max}$ typically grows with the number of nodes $p$, leading to poor accuracy and runtime for the PC-Algorithm. 

Another limitation of the PC-Algorithm is that it requires partial correlations between adjacent nodes to be bounded away from 0; this requirement, which needs to hold for all conditioning sets $S$ such that $|S| \le d_{max}$, is known as \textit{restricted strong faithfulness} \citep{uhler}, and is defined next. 

\begin{definition}\label{def::PCfaithfulness}
Given $\lambda \in (0,1)$, a distribution $P$ is said to be restricted $\lambda$-strong-faithful to a DAG $G = (V,E)$ if the following conditions are satisfied: \\
(i) $\min \left\{ |\rho(X_i, X_j \mid X_S) | : (i,j) \in E, S \subset V \setminus \left\{ i,j \right\}, |S| \le d_{max} \right\} > \lambda$, and \\
(ii) $\min \left\{ |\rho(X_i, X_j \mid X_S) | : (i,j,S) \in N_G \right\} > \lambda$, where $N_G$ is the set of
triples $(i, j,S)$ such that $i, j$ are not adjacent, but there exists $k \in V$ making $(i, j, k)$ an unshielded triple, and $i, j$ are not d-separated given $S$.
\end{definition}
\citet{kalisch} assume that $\lambda = \Omega(n^{-w})$ for $w \in (0, b/2)$ where $b \in (0, 1]$ relates to the scaling of $d_{max}$. In the low-dimensional setting, it has been shown that the PC-Algorithm achieves uniform consistency with $\lambda$ converging to zero at rate $n^{1/2}$, which then gives the same condition as ordinary faithfulness \citep{zhang}. However, in the high-dimensional setting, the set of distributions which are not restricted strong faithful has nonzero measure. In fact, \citet{uhler} showed that this assumption is overly restrictive and that the measure of unfaithful distributions converges to 1 exponentially in $p$. We will revisit the faithfulness assumption in Section~\ref{sec:faithfullness}.

To address the limitations of the PC-Algorithm, we next propose a new algorithm that takes advantage of the structure of large networks from common random graph families. By doing so, we obtain improved computational and sample complexity; as we will show, these complexities are unaffected by the increase in the maximal degree as the graph becomes larger. We also prove consistency under a weaker faithfulness assumption than that needed for the PC-Algorithm. 

\section{The Reduced PC-Algorithm}\label{sec:rpct}

As with the PC-Algorithm, our strategy for estimating the graph skeleton is to start with a complete graph, and then delete edges by discovering separating sets. Under a faithfulness assumption on a linear structural equation model, we do so by computing partial correlations and declaring $X_i$ and $X_j$ d-separated by $S$ if $\rho(X_i, X_j \mid S)$ is smaller than some threshold $\alpha$. Aside from thresholding, the key difference between our proposal and the PC-Algorithm is that we only consider partial correlations conditional on sets $S$ with small cardinality. 

We justify our method using two key observations. Our first key observation is based on the decomposition of covariances over \emph{treks}, which are special types of paths in directed graphs. 

\begin{definition}
A trek between two nodes $i$ and $j$ in a DAG $G$ is either a path from $i$ to $j$, a path from $j$ to $i$, or a pair of paths from a third node $k$ to $i$ and $j$ such that the two paths only have $k$ in common. 
\end{definition}

In a linear SEM \eqref{sem}, the covariance between two random variables is characterized by the treks between them. Denoting a trek between nodes $i$ and $j$ as $\pi: i \leftrightarrow j$ with common node or source $k$, the covariance is given by 
\begin{equation}
\text{cov}(X_i, X_j) = \sum_{\pi: i \leftrightarrow j} \sigma_k \prod_{e \in \pi} \rho_{e}, 
\label{treksum}
\end{equation} 
where $\sigma_k$ is the variance of $\epsilon_k$ from Equation~\eqref{sem} and $\rho_e$ denotes the weight of an edge along the trek, which is the corresponding coefficient in the SEM. This is shown in full detail by \citet{trek}, who also show that the covariance conditional on a d-separating set $S$ leaves out treks which include any nodes in $S$. This conditioning effect is illustrated in Figure~\ref{fig:treks}, where conditioning on an appropriate set $S$ blocks the treks between non-adjacent nodes, without resulting in any additional d-connecting paths. If all edge weights are bounded by 1 in absolute value, then the contribution of each trek to the covariance decays exponentially in trek length. In practice, we scale the data matrix $X$ so that each column has unit standard deviation. Under some conditions, most of the edge weights in the linear SEM \eqref{sem} then satisfy $|\rho_{ij}| < 1$ (this is discussed and shown empirically in Appendix B). Hence, the contribution of long treks to the conditional covariance among non-adjacent nodes is negligible. This decay motivates the thresholding of partial correlations in rPC, and is further discussed in Section~\ref{sec:asymp}.

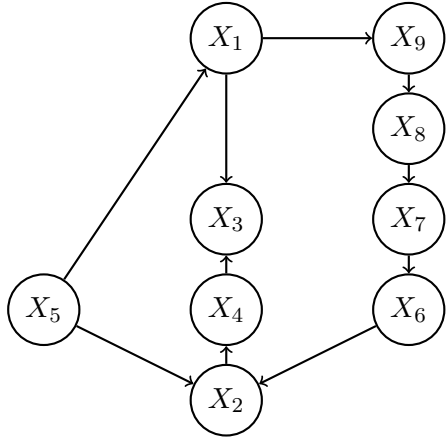
\begin{figure}
\centering
\begin{tikzpicture}[scale=1.2]
\tikzstyle{every node}=[draw,shape=circle];
\tikzstyle{every path}=[thick];
\node (X1) at (0,2) {$X_1$};
\node (X2) at (0,-2) {$X_2$};
\node (X3) at (0,0) {$X_3$};
\node (X4) at (0,-1) {$X_4$};
\node (X5) at (-2,-1) {$X_5$};
\node (X6) at (2,-1) {$X_6$};
\node (X7) at (2,0) {$X_7$};
\node (X8) at (2,1) {$X_8$};
\node (X9) at (2,2) {$X_9$};
\draw[->] (X1) -- (X3);
\draw[->] (X2) -- (X4);
\draw[->] (X4) -- (X3);
\draw[->] (X5) -- (X1);
\draw[->] (X5) -- (X2);
\draw[->] (X1) -- (X9);
\draw[->] (X9) -- (X8);
\draw[->] (X8) -- (X7);
\draw[->] (X7) -- (X6);
\draw[->] (X6) -- (X2);
\end{tikzpicture}
\caption{Illustration of treks between nodes $X_1$ and $X_2$ within a DAG. The middle path involves a collider, $X_3$, so does not contribute to $\text{cov}(X_1, X_2)$, which is $\rho_{51}\rho_{52} + \rho_{19} \rho_{98} \rho_{87} \rho_{76} \rho_{62} $ The conditional covariance $\text{cov}(X_1, X_2 \mid X_5)$ excludes treks that involve $X_5$, and is thus $\rho_{19} \rho_{98} \rho_{87} \rho_{76} \rho_{62}$. Here, $S = \left\{X_5, X_7 \right\}$ is a d-separating set, as it blocks both treks, giving $\text{cov}(X_1, X_2 \mid S) = 0$.}
\label{fig:treks}
\end{figure}

The above observation suggests a new strategy for learning DAG structures by only considering short treks. Suppose $S$ is the set that blocks all short treks between two nodes $j$ and $k$. If the correlation over all remaining d-connecting paths (after conditioning on $S$) between $j$ and $k$ is negligible, then the partial correlation given $S$, $\rho(X_j, X_k \mid S)$ can be used to determine whether $j$ and $k$ are adjacent.

To determine the size of the conditioning set, $S$, we need to determine the number of short treks between any two nodes $j$ and $k$. Our second key observation addresses this question, by utilizing properties of large random graphs. More specifically, motivated by \citeauthor{anandkumar}'s proposal for estimating undirected graphs, we consider a key feature of large random networks, known as the the \emph{local separation property}.
\begin{definition}
Given a graph $G$, a $\gamma$-local separator $S_\gamma(i,j) \subset V$ between non-neighbours $i$ and $j$ minimally separates $i$ and $j$ with respect to paths of length at most $\gamma$. 
\end{definition}

\begin{definition}
A family of graphs $\mathcal{G}$ satisfies the $(\eta,\gamma)$-local separation property if, as $p \rightarrow \infty$, $Pr(\exists G \in \mathcal{G} : \exists (i,j) \not \in E_G, |S_\gamma(i,j)| > \eta ) \rightarrow 0$.
\end{definition}

Intuitively, under $(\eta,\gamma)$-local separation, with high probability, the number of short treks---of length at most $\gamma$---between any two non-neighbouring nodes is bounded above by $\eta$. In fact, as the local separation property refers to any type of path, there are likely even fewer than $\eta$ short treks between any two neighbouring nodes. Therefore, we only need to consider conditioning on sets $S$ of size at most $\eta$ in order to remove the correlation induced by short treks. Combining this with our first insight, we ignore treks (and other possible d-connecting paths) of length longer than $\gamma$, which, for appropriate probability distributions, have a negligible impact on partial correlations. The resulting procedure, called the full reduced PC-Algorithm (rPC-full), is presented in Algorithm~\ref{al1}. 

\begin{algorithm}[b!]
\caption{The full reduced PC-Algorithm (rPC-full)} \label{al1} 
INPUT: Observations from random variables $X_1, X_2, \dots, X_p$; threshold level $\alpha$; maximum separating set size $\eta$. \\
OUTPUT: Estimated skeleton $C$. \\
\begin{tabbing}
   \enspace Set $V = \left\{1, \dots, p \right\}$. \\
   \enspace Form the complete undirected graph $\tilde{C}$ on the vertex set $V$.\\
   \enspace Set $l = -1$; $C = \tilde{C}$. \\
   \enspace \textbf{repeat} \\
   \qquad $l = l+1$ \\
   \qquad \textbf{repeat} \\
   \qquad \qquad Select a (new) ordered pair of nodes $i, j$ that are adjacent in $C$ \\
   \qquad \qquad \textbf{repeat} \\
   \qquad \qquad \qquad Choose (new) $S \subset V \setminus \left\{i,j\right\}$  with $|S| = l$ \\
   \qquad \qquad \qquad \textbf{if} $\rho(X_i, X_j \mid S) \le \alpha$ \\
   \qquad \qquad \qquad \qquad Delete edge $(i,j)$ \\
   \qquad \qquad \qquad \qquad Denote this new graph by $C$ \\
   \qquad \qquad \qquad \textbf{end if} \\
   \qquad \qquad \textbf{until} edge $(i,j)$ deleted or all $S \subset V \setminus \left\{i,j\right\}$ with $|S| = l$ have been chosen \\
   \qquad \textbf{until} all ordered pairs of adjacent nodes $i$ and $j$ have been examined for $\rho(X_i, X_j \mid S) \le \alpha$ \\
   \enspace \textbf{until} $l > \eta$
\end{tabbing}
\end{algorithm}

We note a key implementation difference between rPC-full and the ordinary PC-Algorithm: when searching for a separating set, rPC-full considers all $S \subset V \setminus \left\{i,j\right\}$, while PC-Algorithm only searches over the local neighbourhoods of $i$ and $j$. Recall that if a set $S$ d-separates nodes $i$ and $j$, then there exists a d-separating set $S'$ such that all nodes in $S'$ are adjacent to $i$ or $j$. However, $|S'|$ may be larger than $|S|$; because rPC only considers sets of size up to $\eta$, a full search is needed to ensure discovery of a d-separating set. Motivated by the PC-Algorithm, we also propose an approximate reduced PC-Algorithm (rPC-approx), which uses the same local neighbourhood search, i.e. $S \subset adj(i) \cup adj(j) \setminus \left\{i,j\right\}$.  In practice, we show that rPC-approx performs almost identically to rPC-full.

To recap, our proposal in Algorithm~\ref{al1} hinges on two important properties of probability models on large DAGs: (P1) boundedness of the number of short treks between any two nodes, and (P2) negligibility of correlation over the remaining (long) d-connecting paths. The first property, which is characterized by local separation, concerns solely the DAG structure. \citet{anandkumar} show that many common graph families satisfy the local separation property with small $\eta$. Specifically, sparse, large binary trees, Erd\H{o}s-R{\'e}nyi graphs, and graphs with power law degree distributions all satisfy this property with $\eta \le 2$. Moreover, the sparsity requirement for these graph families is in terms of average node degree, and not the maximum node degree.  For these graph families, our algorithm only needs to consider separating sets of size 0, 1, and 2, irrespective of the maximum node degree. Small-world graphs, as generated by the Watts-Strogatz algorithm \citep{kleinberg2000small}, also satisfy this property, but with $\eta > 2$. In addition, the $\gamma$ parameter increases with $p$ for these families; thus, as graphs get larger, the local separation property applies to a larger set of paths. 

By only considering a bounded number of short paths, our algorithm has computational complexity $O(p^{\eta + 2})$, and thus avoids the exponential scaling in $d_{max}$ that the PC-Algorithm suffers from. This is particularly significant in the case of power-law graphs, where $d_{max} = O(p^a)$ for $a > 0$ \citep{MolloyReed1995}; in this case, PC-Algorithm has a computational complexity of $O(p^{p^{a}})$, which is significantly worse than rPC's complexity of $O(p^{4})$. While rPC-full might not be faster in practice due to the larger search space, rPC-approx would result in significant speedup. 

Unlike the first property (P1), the second property needed for our algorithm, namely the negligibility of correlation over d-connecting paths, concerns both the structure of the DAG $G$, and the probability distribution $P$ of variables on the graph. In the next section, we discuss two alternative sufficient conditions that guarantee this property, and allow us to consistently estimate the DAG skeleton. 

\section{Algorithm Analysis and Asymptotics}\label{sec:asymp}

In this section, we describe in detail the assumptions required for the rPC-full algorithm to consistently recover the DAG skeleton. We also discuss its computational and statistical properties, particularly in comparison with the PC-Algorithm.   

\subsection{Consistency}

As discussed in the previous section, to consistently recover the DAG skeleton, rPC requires that properties (P1) and (P2) hold; namely, that the graph under consideration has a bounded number of short paths and the correlation over d-connecting paths of length greater than $\gamma$ decays sufficiently quickly. In fact, the trek decomposition \eqref{treksum} indicates that the correlation for each long trek is small when most of the edge weights are bounded by 1 in absolute value. However, condition (P2) requires the total correlation over all long treks (and other d-connecting paths) to decay sufficiently quickly. To this end, we consider two alternative sufficient conditions. The first condition is a direct assumption on the boundedness of the conditional correlation. The second is inspired by \citet{anandkumar}, and assumes the underlying probability model satisfies what we term \emph{directed walk-summability}. This condition mirrors the walk-summability condition for undirected graphical models, which has been well-studied and shown to hold in a large class of models \citep{walksums}. 

\begin{definition}
A probability model is directed $\beta$-walk-summable on a DAG with weighted adjacency matrix $A$, if $\| A \| \le \beta < 1$ where $\| \cdot \|$ denotes the spectral norm.
\end{definition}

As an alternative to this condition, we also present a direct assumption, Assumption~\ref{assum::bddtrek}. This condition is less restrictive than directed walk-summability, but has not been characterized in the literature. However, given that $\gamma$ increases with $p$ in the graph families we are considering, it is intuitive that the sum of edge weight products over treks longer than length $\gamma$ will be decreasing and asymptotically small. This also holds for non-trek d-connecting paths. These two assumptions lead to two parallel proofs of the consistency of our algorithm, presented in Theorem~\ref{thm::consistency}. Before stating the theorem, we discuss our assumptions. 

Similar to the PC-Algorithm, our method requires a faithfulness condition. As stated previously, our condition, which we term \emph{path faithfulness} and is defined next, is weaker than PC-Algorithm's $\lambda$-strong faithfulness stated in Definition~\ref{def::PCfaithfulness} (see Section~\ref{sec:faithfullness} for additional details). 
\begin{definition}\label{def::ourfaithfulness}
Given $\lambda \in (0,1)$, a distribution $P$ is said to be $\lambda$-path-faithful to a DAG $G = (V,E)$ if both of the following conditions hold: \\
(i) $\min \left\{ |\rho(X_i, X_j \mid X_S) | : (i,j) \in E, S \subset V \setminus \left\{ i,j \right\}, |S| \le \eta \right\} > \lambda$, for some $\eta$, and \\
(ii) $\min \left\{ |\rho(X_i, X_j \mid X_S) | : (i,j,S) \in N_G \right\} > \lambda$, where $N_G$ is the set of
triples $(i, j,S)$ such that $i, j$ are not adjacent, but there exists $k \in V$ making $(i, j, k)$ an unshielded triple, and $i, j$ are not d-separated given $S$.
\end{definition}
Part (i) of the assumption only requires partial correlations between true edges conditioned on sets of size up to $\eta$ to be bounded away from zero, while the PC-Algorithm requires this for conditioning sets of size up to $d_{max}$. In Section~\ref{sec:faithfullness}, we discuss how this affects bounds on the true partial correlations, and also empirically show that the above path faithfulness assumption is less restrictive than corresponding assumption for the PC-Algorithm. 

\begin{assumption}[Path faithfulness and Markov property]\label{assum::faith}
The probability distribution $P$ of random variables corresponds to a linear SEM \eqref{sem} with sub-Gaussian errors, and is $\lambda$-path-faithful to the DAG $G$, with $\lambda = \Omega(n^{-c})$ for $c \in (0, 1/2)$.
\end{assumption}

Our second assumption ensures that the covariance matrix of the structural equation model and its inverse remain bounded as $p$ grows. 

\begin{assumption}[Covariance and precision matrix boundedness]\label{assum::cov}
The covariance matrix of the model $\Sigma_G$ and its inverse $\Sigma_G^{-1}$ are bounded in spectral norm, that is, $\max(\| \Sigma_G \|, \| \Sigma_G^{-1} \|) \le M < \infty$ for all $p$. 
\end{assumption}

The last three assumptions characterize applicable graph families and probability distributions. 

\begin{assumption}[$(\eta,\gamma)$-local separation]\label{assum::locsep}
The DAG $G$ belongs to a family of random graphs $\mathcal{G}$ that satisfies the $(\eta,\gamma)$-local separation property with $\eta = O(1)$ and $\gamma = O( \log p)$. 
\end{assumption}

\begin{assumption}[Bounded long path weight]\label{assum::bddtrek}
Let $\pi$ denote a d-connecting path of length $l(\pi)$ between two non-adjacent vertices $i$ and $j$ in $G$. Then, there exists a conditioning set $S$ such that the total edge weight over d-connecting paths longer than $\gamma$ satisfies:
\begin{equation*}
\sum_{ l = \gamma + 1}^{p-1} \sum_{l(\pi) = l} | \rho_{\pi, 1} \dots \rho_{\pi, l} | = O(\beta^\gamma), 
\end{equation*}
for some $\beta \in (0,1)$.
\end{assumption}

Assumption~\ref{assum::bddtrek} guarantees that the sum of weights over long treks between any two nodes $i$ and $j$ is bounded. For a single trek, a sufficient condition is that all edge weights are bounded by 1 in magnitude. In Appendix~B, we provide further discussion on Assumption~\ref{assum::bddtrek}, and empirically investigate its plausibility. In particular, we show that if the data matrix $X$ is scaled so that each column has unit standard deviation, then with high probability all edge weights are bounded by 1 in absolute value. To account for residual correlation induced through conditioning on common descendants of $i$ and $j$, Assumption~\ref{assum::bddtrek} also includes non-trek d-connecting paths. 

\begin{assumption}[Directed $\beta$-walk-summability]\label{assum::walksum}
The probability distribution $P$ is directed $\beta$-walk-summable. 
\end{assumption}

We are now ready to state our main result. The result is proved in Appendix A, where the error probabilities are also analyzed. 

\begin{theorem}\label{thm::consistency}
Under Assumptions~\ref{assum::faith}-\ref{assum::locsep} and either Assumption~\ref{assum::bddtrek} or \ref{assum::walksum}, there exists a parameter $\alpha$ for thresholding partial correlations such that, as $n,p \longrightarrow \infty$ with $n = \Omega \{ (\log p)^{1/(1 - 2c)} \} $, the full reduced PC (rPC-full) procedure, as described in Algorithm~\ref{al1}, consistently learns the skeleton of the DAG $G$.
\end{theorem}

Several theoretical features of our algorithm are attractive. As stated previously, our faithfulness condition is weaker than the corresponding assumption for the PC-Algorithm and related methods. Similar to its computational complexity, the sample complexity of our algorithm also does not scale with the maximal node degree, and is only dependent on the parameter $\eta$ as $p$ increases. For example, in a power law graph, the sample complexity of the PC-Algorithm is $\Omega\{ \max(\log p, p^{ab}) \}$ for $0 < a, b < 1$, compared to $\Omega \{ (\log p)^{1/1 - 2c} \}$ with $c \in (0, 1/2)$ for our method. This gain in efficiency is due to fact that the maximum separating set size, $\eta$, remains constant in rPC. Finally, our algorithm does not require the data to be jointly Gaussian. The proof of the algorithm's consistency only requires that the population covariance matrix can be well-approximated from the data; for simplicity, we assume a sub-Gaussian distribution.

\subsection{On Faithfulness Assumption}
\label{sec:faithfullness}

As stated in Section~\ref{sec:pc}, for large biological networks of interest, the maximum node degree, $d_{max}$, often grows with $p$. Therefore, the $\lambda$-restricted strong faithfulness condition of the PC-Algorithm---Definition~\ref{def::PCfaithfulness}---becomes exponentially harder to satisfy with increasing network size. A full discussion of this phenomenon can be found in \citet{uhler}, where it is shown that the measure of strong unfaithful distributions converges to 1 for various graph structures. Although this would also occur with path faithfulness (Definition~\ref{def::ourfaithfulness}), our condition allows a rate for $\lambda$ that is independent of $d_{max}$ and $p$. 

The rate for $\lambda$ in the PC-Algorithm is $\lambda = \Omega(n^{-w})$ for $w \in (0, b/2)$, where $d_{max} = O(n^{1-b}) \mbox{ for } b \in (0, 1]$. For $b = 1$, or constant $d_{max}$, the PC-Algorithm's required scaling for $\lambda$ is identical to that for our method in Assumption~\ref{assum::faith}. This makes sense intuitively, since our method is not affected by the increase in $d_{max}$. For other values of $b$, the scaling of $\lambda$ becomes more restricted for the PC-Algorithm; for example, if $b = 1/2$, then $\lambda = \Omega(n^{-w})$ for $w \in (0, 1/4)$. However, under path faithfulness (Definition~\ref{def::ourfaithfulness}), we can still achieve a rate of $\lambda = \Omega(n^{-1/2})$; that is, the partial correlations are allowed to be smaller and the condition is weaker.

We report the findings of a simulation study, similar to that in \citet{uhler}, which examines how often randomly generated DAGs satisfy part (i) of the path faithfulness assumption compared to restricted strong faithfulness. We are primarily interested in part (i), as this part is needed for consistent skeleton estimation; part (ii), on the other hand, is needed to obtain correct separating sets in order to obtain partial orientation of edges. In this simulation, 1000 random DAGs were generated from Erd\H{o}s-R{\'e}nyi and power law families, with edge weights drawn independently from a $\text{Uniform}(-1, 1)$ distribution. Each DAG had $p = 20$ nodes, with varying expected degrees per node. For each simulation setting, we computed the proportion of DAGs that satisfied part (i) of the $\lambda$-restricted-strong-faithfulness and $\lambda$-path-faithfulness conditions with $\lambda = 0.001$ and $\eta = 2$. The results are shown in Table~\ref{tab:faithsim}. We see that path faithfulness is much more likely to be satisfied than restricted strong faithfulness, especially for power law graphs. This is to be expected, as the number of constraints required for restricted strong faithfulness grows with $d_{max}$, but remains constant for path faithfulness. It is, however, difficult for dense graphs to satisfy either condition, although there is a mild advantage for path faithfulness. In Appendix B, we provide the results of further simulation studies with $p = 10$ and $p = 30$; both of these give similar conclusions and indicate that path faithfulness remains easier to satisfy with increasing network size. 

\begin{table}
\caption{\label{tab:faithsim} Empirical probabilities of random DAGs of size $p = 20$ satisfying faithfulness conditions; RSF refers to restricted strong faithfulness of the PC-Algorithm, and PF refers to path faithfulness of reduced PC (rPC).}
\centering
\begin{tabular}{cccc}
Graph family & Expected degree & $Pr(\text{RSF})$ & $Pr(\text{PF})$ \\
\hline
Erd\H{o}s-R{\'e}nyi           & 2               & 0.77     & 0.92    \\
Erd\H{o}s-R{\'e}nyi           & 5               & 0      & 0.08     \\
Power law           & 2               & 0.54     & 0.85    \\
Power law           & 6               & 0.003      & 0.08   
\end{tabular}
\end{table}

\subsection{Tuning Parameter Selection}
\label{sec:tuning}

Our algorithm requires two tuning parameters: the maximum separating set size $\eta$, and the threshold level for partial correlations $\alpha$. The parameter $\eta$ varies based on the underlying graph family. Thus, given knowledge of a plausible graph structure, $\eta$ can be pre-specified. Alternatively, $\eta$ can be selected by maximizing a goodness-of-fit score over a parameter grid, along with $\alpha$. This may be preferable as the local separation results consider all short paths, not just treks, so better performance may be obtained by specifying a smaller $\eta$. Likewise, when using the rPC-approx algorithm, a larger $\eta$ could be needed to discover appropriate separating sets. 

For jointly Gaussian data, we can obtain a modified version of the Bayesian information criterion by fitting the likelihood to the CPDAG obtained based on the estimated DAG skeleton \citep{foygel}. Following \citet{anandkumar}, and denoting by $X_{obs}$ the observed data, we use: 
\begin{equation} \label{eqn::bic}
\textsc{bic}(X_{obs}; \hat{G}) = \log f(X_{obs}; \hat{\theta}) - 0.5|E|\log(n) - 2|E|\log(p), 
\end{equation}
where $\hat{G}$ denotes one of the DAGs obtained from the estimated CPDAG containing $|E|$ edges. The CPDAG represents the Markov equivalence class of DAGs, so all possible graphs will result in the same fitted Gaussian model with parameters $\hat{\theta}$. We use this \textsc{bic} for tuning parameter selection; higher scores imply a better fit. For linear SEMs with non-Gaussian noise distributions, the Gaussian likelihood serves as a surrogate goodness-of-fit measure, and  \textsc{bic} can still be used to select the tuning parameters.

\section{Simulation Studies}
\label{sec:sims}

In this section, we compare the performance of rPC-approx, rPC-full, and the standard PC-Algorithm in multiple simulation settings. We consider both setting a constant value for $\eta$ in rPC, and tuning it to maximize the \textsc{bic}. 

\subsection{Pre-specified $\eta$ Parameter}

To facilitate the comparison with the PC-Algorithm, we generate data from Gaussian linear SEMs as in Equation~\ref{sem}, with the dependency structure specified by a DAG from Erd\H{o}s-R{\'e}nyi and power law families. We implement our algorithm with maximum separating set size $\eta = 2$ since these families are known to satisfy $(\eta,\gamma)$-local-separation with $\eta \le 2$ \citep{anandkumar}. 

We generate a random graph with $p$ nodes using the \texttt{igraph} library in R, assigning every edge a weight from a $\text{Uniform}(0.1,1)$ distribution. We then use the \texttt{rmvDAG} function from the \texttt{pcalg} library to simulate $n$ observations from the DAG. This is repeated 20 times for each  thresholding level $\alpha$; average true and false positive rates for both algorithms are reported over the grid of $\alpha$ values. 
Our grid of $\alpha$ values produces partial receiver operating characteristic (pROC) curves for varying sample sizes and graph structures, which are used to assess the estimation accuracy of the methods. 

For both Erd\H{o}s-R{\'e}nyi and power law DAGs, we consider a low-dimensional setting with $p = 100$ nodes and $n = 200$ observations, and two high-dimensional settings with $(p, n) = (200, 100)$ and $(p, n) = (500, 200)$. In all settings, the DAGs are set to have an average degree of 2. The maximum degrees of the Erd\H{o}s-R{\'e}nyi graphs range from 5 to 7. The maximum degrees for power law graphs increase with $p$ and are 42, 69, and 71 for the three simulation settings. 

Results for Erd\H{o}s-R{\'e}nyi DAGs are shown in Figure~\ref{fig:ER}. In this setting, all algorithms perform almost identically in both low and high-dimensional cases. This is because for the relatively small maximum degree in Erd\H{o}s-R{\'e}nyi graphs, the sizes of conditioning sets for PC and rPC are not very different. 

\begin{figure}
\centering
\includegraphics[height=0.24\textheight]{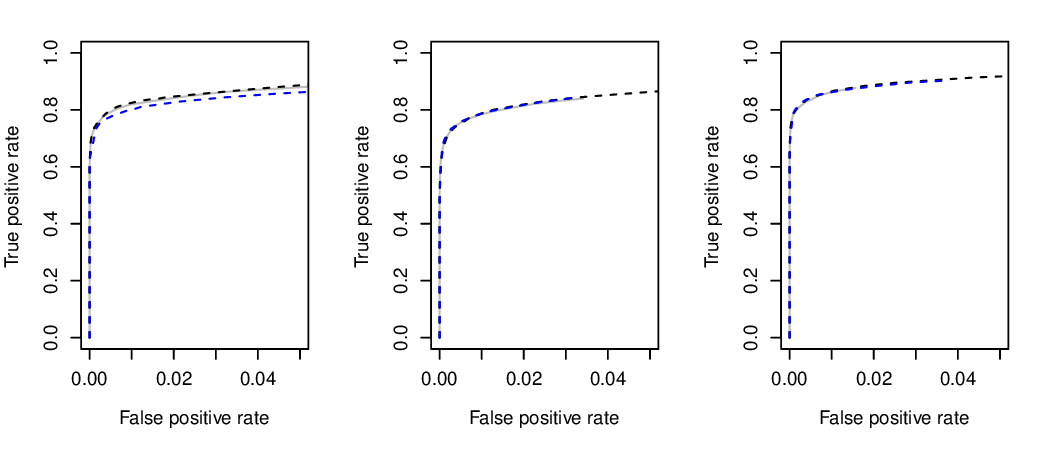}
\caption{Average true vs. false positive rates for PC-Algorithm (grey solid line), rPC-approx (black dashed line), and rPC-full (blue dashed line) estimating Erd\H{o}s-R{\'e}nyi DAGs. Left: $p=100, n = 200$; centre: $p = 200, n = 100$; right: $p = 500, n = 200$.}
\label{fig:ER}
\end{figure}  
\begin{figure}
\centering
\includegraphics[height=0.24\textheight]{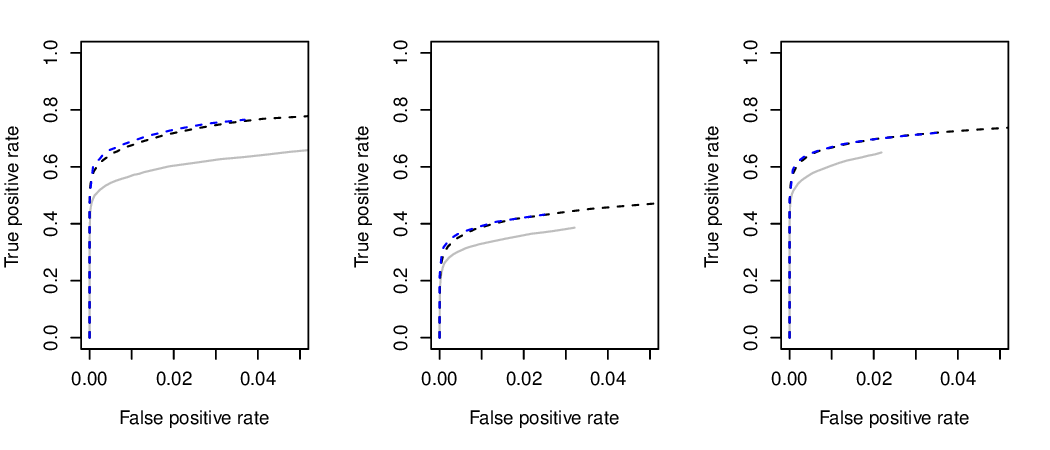}
\caption{Average true vs. false positive rates for PC-Algorithm (grey solid line), rPC-approx (black dashed line), and rPC-full (blue dashed line) estimating power law DAGs. Left: $p=100, n = 200$; centre: $p = 200, n = 100$; right: $p = 500, n = 200$.}
\label{fig:PL}
\end{figure}  
\begin{figure}
\centering
\includegraphics[height=0.24\textheight]{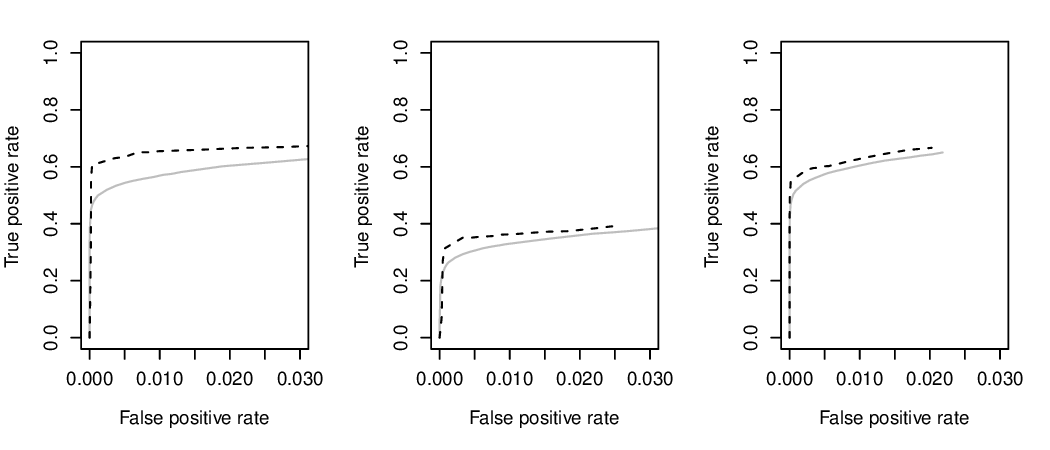}
\caption{Average true vs. false positive rates for PC-Algorithm (grey solid line) and rPC-approx with \textsc{bic}-tuned $\eta$ (black dashed line) estimating power law DAGs. Left: $p=100, n = 200$; centre: $p = 200, n = 100$; right: $p = 500, n = 200$.}
\label{fig:bic}
\end{figure}  

Results for power law DAGs are shown in Figure~\ref{fig:PL}. In this case, both versions of rPC offer improved accuracy compared to the PC-Algorithm, in both low and high-dimensional settings. These results confirm our theoretical findings, and show that our algorithm performs better at estimating DAGs with hub nodes than the PC-Algorithm. 

Additional simulations in Appendix B show similar results in more dense DAGs. As the underlying DAG becomes more dense, all methods perform worse; however, our algorithms maintain an advantage over the PC-Algorithm in the power law setting. All of these simulations also confirm that rPC-full and rPC-approx perform very similarly. This suggests that the situation where conditioning on non-local sets is required occurs with very low probability, indicating that rPC-approx provides suitably good estimates. 

We also compare the runtimes of the algorithms for these settings. Over 100 iterations, we generate a random dataset, and apply all the algorithms with a range of tuning parameters. Specifically, we set $\eta = 2$ for rPC, and $\alpha = \left\{10^{-6}, 10^{-5}, 10^{-4}, 10^{-3}, 10^{-2}\right\}$. We then take the total runtime over all parameters, and compare by considering the mean value of $100 \left( \frac{time_{rPC}}{time_{PC}} \right)$ for both rPC-approx and rPC-full. The results are shown in Table~\ref{tab:time}. We observe that rPC-approx is significantly faster than the PC-Algorithm for power law graphs. As expected, our implementation of rPC-full is slower than the PC-Algorithm in all settings, given its exhaustive search over all possible separating sets. Given the results indicating that rPC-approx performs almost identically to rPC-full, we suggest that practitioners use rPC-approx, possibly combined with \textsc{bic} tuning if a suitable $\eta$ bound is not known. We present relevant simulation results in the next section.

\begin{table}[ht]
\caption{\label{tab:time} Empirical ratio of rPC-approx and rPC-full runtimes to PC-Algorithm runtime.}
\centering
\begin{tabular}{ccccc}
Graph family & $p$   & $n$   & rPC-approx & rPC-full \\
  \hline
 Erd\H{o}s-R{\'e}nyi & 100 & 200 & 0.98 & 3.08 \\ 
 Erd\H{o}s-R{\'e}nyi & 200 & 100 & 1.00 & 6.16 \\ 
 Erd\H{o}s-R{\'e}nyi & 500 & 200 & 1.00 & 23.25 \\ 
 Power law & 100 & 200 & 0.60 & 1.13 \\ 
 Power law & 200 & 100 & 0.92 & 3.81 \\ 
 Power law & 500 & 200 & 0.36 & 5.23 \\ 
   \hline
\end{tabular}
\end{table}

\subsection{BIC-tuned $\eta$ Parameter}

In this section, we consider simulations where the maximum size of the conditioning set for rPC-approx, $\eta$, is selected to maximize the \textsc{bic} score defined in \eqref{eqn::bic}. To this end, we consider power law DAGs with the same low and high-dimensional settings as before. We select the value of $\eta \in \{1,2,3,4\}$ which maximizes the \textsc{bic} at each $\alpha$ value. While rPC-full could also be tuned in this way, it would be computationally expensive to consider $\eta$ beyond 3 for $p > 200$. The results in Figure~\ref{fig:bic} show that our algorithm maintains an advantage over the PC-Algorithm. Interestingly, for values of $\alpha$ which yielded the best estimation accuracy, the optimal $\eta$ selected was most frequently 1, which confirms our intuition from Section~\ref{sec:rpct} that the $\eta$ parameter for a graph family should be seen as an upper bound for estimating DAGs using our algorithm. 

\section{Application: Estimation of Gene Regulatory Networks}
\label{sec:appl}

We apply our algorithms and the PC-Algorithm to a gene expression data set of $n = 487$ patients with prostate cancer from The Cancer Genome Atlas \citep{tcga}. We select $p = 272$ genes with known network structure from BioGRID \citep{biogrid}, and attempt to recover this network from the data. We choose the tuning parameters for both rPC algorithms and the p-value threshold for the PC-Algorithm by searching over a grid of values and selecting those which yielded the largest \textsc{bic} \eqref{eqn::bic}. We found that the best rPC algorithm in terms of \textsc{bic} was rPC-approx with $\eta = 3$ and $\alpha = 0.09$, while the best p-value threshold for the PC-Algorithm was 0.07. 

The BioGRID database provides valuable information about known gene regulatory interactions. However, this databse mainly capture genetic interactions in normal cells. Thus, the information from BioGRID may not correctly capture interactions in cancerous cells, which are of interest in our application \citep{ideker2012differential}. Despite this limitation, highly connected hub genes in the BioGRID network, which usually correspond to  transcription factors, are expected to stay highly connected in cancer cells. Therefore, to evaluate the performance of the two methods, we focus here on the identification of hub genes, which are often most clinically relevant \citep{hub-imp1, hub-imp2, hub-imp3}. 

The two estimated networks and their hub genes are visualized in Figure~\ref{fig:hubs}. 
Here, we define hub genes as nodes with degree at least 8, which corresponds to the 75th percentile in the degree distribution of both estimates. rPC-approx identifies 19 of 57 true hubs, while the PC-Algorithm only identifies 6. Interestingly, several of the hub genes uniquely identified by rPC are known to be associated with prostate cancer, including ACP1, ARHGEF12, CDH1, EGFR, and PLXNB1 \citep{cancergenes1, cancergenes2, cancergenes3, cancergenes4, cancergenes5}. These results suggest that rPC may be a promising alternative for estimating biological networks, where highly-connected nodes are of clinical importance. 

Examining the two networks also indicates that for nodes with small degrees, the estimated neighborhoods from rPC are very similar to those from the PC-Algorithm. To assess this observation, we consider the induced subgraph of nodes with degree at most 5 in the PC-Algorithm estimate. The $F_1$ score---which is a weighted average of precision and recall---between the two estimates of this sparse subnetwork is 0$\cdot$86. This value indicates that the two algorithms perform very similarly over sparse nodes.

\begin{figure}
\centering
\includegraphics[width=0.95\textwidth]{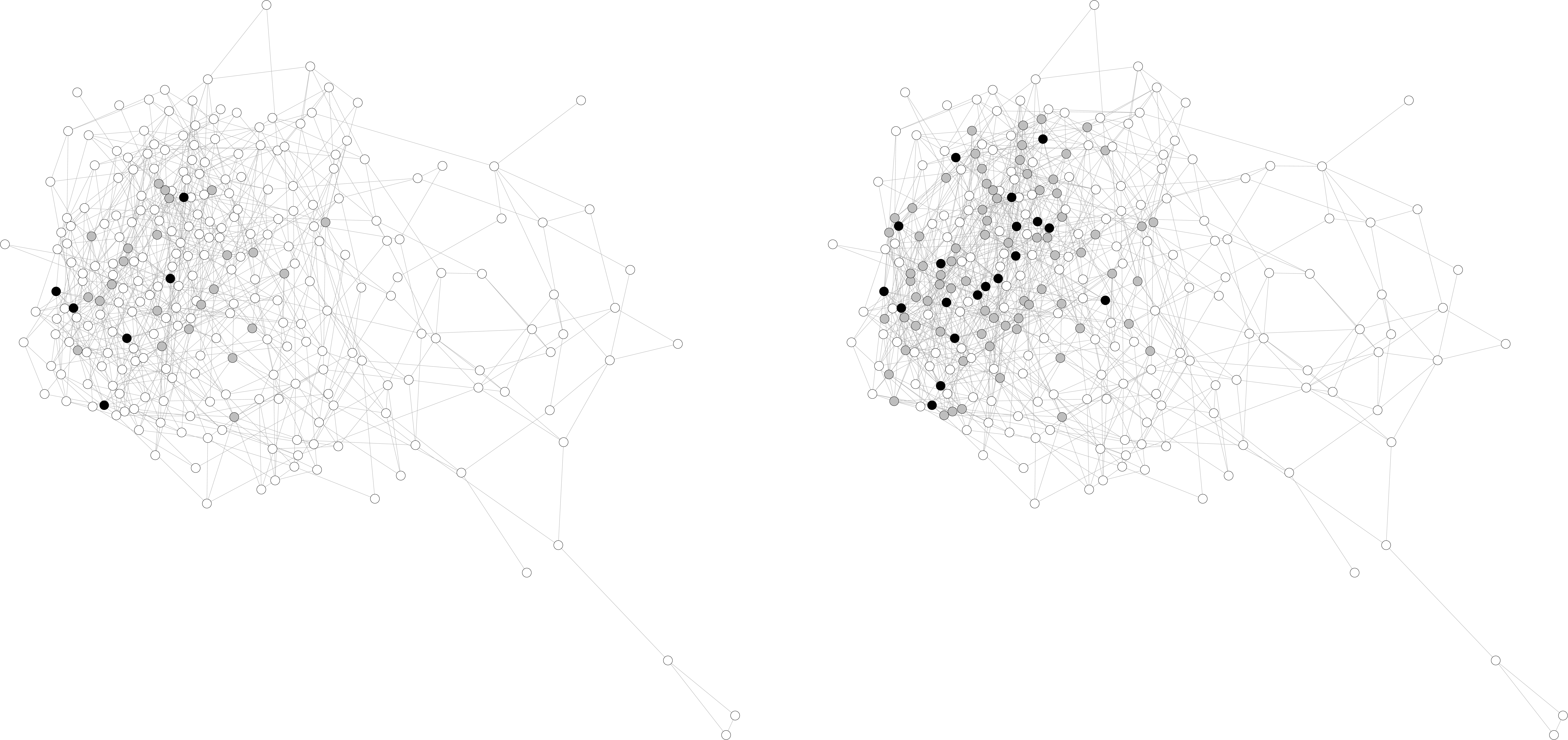}
\caption{Estimated skeletons of gene regulatory networks in prostate cancer subjects. Black nodes are classified as hubs, having estimated degree of at least 8. Grey nodes are identified hubs that are also considered hubs in the BioGRID data. Left: PC-Algorithm; right: rPC-approx.}
\label{fig:hubs}
\end{figure}  

\section{Discussion}\label{sec:disc}
Our new algorithm for learning directed acyclic graphs (DAGs) by conditioning on small sets leads to more efficient computation and estimation under a less restrictive faithfulness assumption than the PC-Algorithm. However, our weaker faithfulness condition may still not be satisfied for dense DAGs or in structural equation models with edge weights distributed over a larger parameter space. This is shown both geometrically and empirically in \citet{uhler}, and remains a direction for future research. Generalizing the idea of restricted conditioning to more complex probability models over DAGs, such as nonlinear SEMs \citep{spacejam} would also be of interest. Finally, the idea of conditioning on small sets of variables can also be used to develop more efficient hybrid methods for learning DAGs in high dimensions.


\newpage

\appendix
\section*{Appendix A. Proof of Theorem~\ref{thm::consistency}}
\label{app:theorem}

In this section, we prove the consistency of rPC-full for estimating the skeleton of a directed acyclic graph under the assumptions stated in the main paper.

We begin by establishing that correlations decay over long paths in the graph, and use this to show that the partial correlation between two non-adjacent nodes, conditional on a suitable set $S$ with small cardinality, is bounded above. We establish this through two possible sufficient conditions: Lemma 1 is based on Assumption~\ref{assum::bddtrek}, which directly assumes that the the total weight over long paths is sufficiently small; on the other hand, Lemma 2 uses Assumption~\ref{assum::walksum}, which assumes the underlying model is directed walk-summable. Combining this result with Assumption~\ref{assum::faith}, which says that the relevant partial correlations between two adjacent nodes is bounded below, we have oracle consistency of rPC-full. In Lemma 3, we invoke a concentration inequality for sample partial correlations to bound their deviations from population quantities. Using these results, we then prove that there exists a threshold level that consistently recovers the true skeleton in the finite sample setting. \\

In our first two lemmas, we make use of the local separation property in Assumption~\ref{assum::locsep}. While Assumption~\ref{assum::locsep} does not directly concern \emph{local separating sets}, since a d-separating set is a subset of a general separating set, Assumption~\ref{assum::locsep} asymptotically guarantees the existence of a \textit{$\gamma$-local d-separator} of size at most $\eta$ for any two non-neighbouring nodes. 

\begin{definition}
Given a graph $G$, a $\gamma$-local d-separator $S_\gamma(i,j) \subset V \setminus \left\{i,j\right\}$ between non-neighbours $i$ and $j$ minimally d-separates $i$ and $j$ over paths of length at most $\gamma$. 
\end{definition}

\begin{lemma}
Under Assumptions 1-4, the partial correlation between non-neighbours $i$ and $j$ satisfies 
\[
\min\limits_{ \substack{S \in S_{\eta, \gamma} } } |\rho(i,j \mid S)| = O(\beta^\gamma).
\] 
where $S_{\eta, \gamma}$ is the set of $\gamma$-local d-separators of size at most $\eta$. 
\end{lemma}

\begin{proof}

Recall the form of the linear structural equation model: 
\begin{equation*}
X_k = \sum_{j \in pa(k)} \rho_{jk} X_j + \epsilon_k,
\end{equation*}
where $\epsilon_k$ are independent and $Var(\epsilon_k) = \sigma_k^2 < \infty$ for all $k$. 
 
Let $A_G$ denote the lower-triangular weighted adjacency matrix for the graph $G$, obtained by ordering the nodes according to a causal order \citep{shojaie}, so that $j \in pa(k)$ implies $j < k$. Then, as shown by \citet{shojaie} for the Gaussian case and by \citet{loh-buhlmann} in general linear structural equation models, 
\begin{equation*}
\Sigma_G = (I - A_G)^{-1} D (I - A_G)^{-T},
\end{equation*}
where $D = $ diag$(\sigma^2_1,...,\sigma^2_p)$.

First, suppose that $\sigma^2_i = 1$ for all $i$. We consider the conditional covariance $\Sigma_G(i,j \mid S)$ where $i$ and $j$ are non-neighbours, and $S$ is the $\gamma$-local d-separator, as defined above. Let $\pi$ denote a d-connecting path between $i$ and $j$ of length $l(\pi)$, and $\rho_1, \ldots, \rho_l$ denote the edge weights along the path. By conditioning on $S$, we have that covariance is only induced through d-connecting paths of length greater than $\gamma$, as referenced in the main paper. Then,

\begin{align*}
\Sigma_G(i,j \mid S) = \sum_{ \substack{ \pi: i \leftrightarrow j \\ \pi \cap S = \emptyset }} \sum_{\pi: l(\pi) = l} \prod_{k = 1}^l \rho_{k}  = \sum_{ \substack{ \pi: i \leftrightarrow j \\ l(\pi) = \gamma+1}}^{p-1} \sum_{\pi: l(\pi) = l} \prod_{k = 1}^l \rho_{k}.
\end{align*}

Therefore:
\begin{align*}
| \Sigma_G(i,j \mid S) | &= \Bigg| \sum_{ \substack{ \pi: i \leftrightarrow j \\ l(\pi) = \gamma+1}}^{p-1} \sum_{\pi: l(\pi) = l} \prod_{k = 1}^l \rho_{k} \Bigg| \\
&\le \sum_{ \substack{ \pi: i \leftrightarrow j \\ l(\pi) = \gamma + 1}}^{p-1} \sum_{\pi: l(\pi) = l} \prod_{k = 1}^l | \rho_{k} | &&\text{(by triangle inequality)} \\
&= O(\beta^\gamma). &&\text{(by Assumption 4)} 
\end{align*}

Now, suppose that not all $\sigma^2_i = 1$. Then, let $\sigma^2_{max} = \max_i \sigma^2_i$. We have:
\begin{align*}
| \Sigma_G(i,j \mid S) | 
&\le \sum_{ \substack{ \pi: i \leftrightarrow j \\ l(\pi) = \gamma + 1}}^{p-1} \sum_{\pi: l(\pi) = l} \sigma^2_{max}  \prod_{k = 1}^l | \rho_{k} |  \\
&= \sigma^2_{max} \sum_{ \substack{ \pi: i \leftrightarrow j \\ l(\pi) = \gamma + 1}}^{p-1} \sum_{\pi: l(\pi) = l} \prod_{k = 1}^l | \rho_{k} | \\
&= \sigma^2_{max} O(\beta^\gamma) \\
&= O(\beta^\gamma). &&\text{(by Assumption 2)}
\end{align*}

Finally, we have $|\rho(i,j \mid S)| = \cfrac{|\Sigma_G(i,j \mid S)|}{\sqrt{\Sigma_G(i,i \mid S) \Sigma_G(j,j \mid S)}} = O(\beta^\gamma)$ by Assumption 2, since the conditional variances are functions of the marginal variances, which are bounded. 

\end{proof}

Next, we show the same result by assuming directed walk-summability of the model. 

\begin{lemma}
Under Assumptions 1-3 and Assumption~\ref{assum::walksum}, the partial correlation between non-neighbours $i$ and $j$ satisfies 
\[
\min\limits_{ \substack{S \in S_{\eta, \gamma} } } |\rho(i,j \mid S)| = O(\beta^\gamma).
\] 
where $S_{\eta, \gamma}$ is the set of $\gamma$-local d-separators of size at most $\eta$. 
\end{lemma}

\begin{proof}

Recall from the proof of Lemma 1, 
\begin{equation*}
\Sigma_G = (I - A_G)^{-1} D (I - A_G)^{-T},
\end{equation*}
where $D = $ diag$(\sigma^2_1,...,\sigma^2_p)$. First, suppose that $\sigma^2_i = 1$ for all $i$. Then, we can write:

\begin{align*}
\Sigma_G &= \left(\sum_{r=0}^\infty A_G^r \right) \left( \sum_{r=0}^\infty A_G^r \right)^T \\
&= \left(\sum_{r=0}^\gamma A_G^r + \sum_{r=\gamma+1}^\infty A_G^r  \right) \left( \sum_{r=0}^\gamma A_G^r + \sum_{r=\gamma+1}^\infty A_G^r  \right)^T . \\
\end{align*}

Now, let $\Sigma_H$ denote the covariance matrix induced by only considering d-connecting paths of length at most $\gamma$. For convenience, let $\Lambda_H := \sum_{r=0}^\gamma A_G^r$ and $R_\gamma :=  \sum_{r=\gamma+1}^\infty A_G^r $. Considering their spectral norms, denoted by $\| \cdot \|$, we have by walk-summability that $\| \Lambda_H \| \le  \dfrac{1 - \beta^{\gamma + 1}}{1-\beta}$ and $\| R_\gamma \| \le  \dfrac{\beta^{\gamma + 1}}{1-\beta}$. Then,

\begin{align*}
\Sigma_G &= ( \Lambda_H +  R_\gamma ) ( \Lambda_H +  R_\gamma )^T \\
&=  \Lambda_H  \Lambda^T_H + \Lambda_H R_\gamma^T + R_\gamma \Lambda_H^T + R_\gamma R_\gamma^T \\
&=  \Sigma_H + \Lambda_H R_\gamma^T + R_\gamma \Lambda_H^T + R_\gamma R_\gamma^T
\end{align*}

Now, defining $E_\gamma := \Sigma_G - \Sigma_H$ and taking spectral norms, we get:
\begin{align*}
\| E_\gamma \| = \| \Sigma_G - \Sigma_H \| 
&= \| \Lambda_H R_\gamma^T + R_\gamma \Lambda_H^T + R_\gamma R_\gamma^T \| \\
&\le \| \Lambda_H R_\gamma^T \| + \| R_\gamma \Lambda_H^T \| + \| R_\gamma R_\gamma^T \| \\
&\le 2\| \Lambda_H \| \| R_\gamma \| + \| R_\gamma \|^2 \\
&\le 2 \bigg( \dfrac{1 - \beta^{\gamma + 1}}{1-\beta} \bigg) \bigg( \dfrac{\beta^{\gamma + 1}}{1-\beta} \bigg) + \bigg( \dfrac{\beta^{\gamma + 1}}{1-\beta} \bigg)^2 \\
&= \dfrac{2\beta^{\gamma + 1} - 2\beta^{2\gamma + 2}}{(1-\beta)^2} +  \dfrac{\beta^{2\gamma + 2}}{(1-\beta)^2} \\
&= \dfrac{2\beta^{\gamma + 1} - \beta^{2\gamma + 2}}{(1-\beta)^2}  \\
&= \dfrac{\beta^{\gamma + 1} (2 - \beta^{\gamma + 1})}{(1-\beta)^2}  \\
&= O(\beta^\gamma).
\tag{S1}
\end{align*}

Now, suppose that not all $\sigma^2_i = 1$. Then, following the same expansion of $\Sigma_G$ as above, we have:
\begin{equation*}
\| E_\gamma \| = \|D\| \dfrac{\beta^{\gamma + 1} (2 - \beta^{\gamma + 1})}{(1-\beta)^2} = \sigma^2_{max} \dfrac{\beta^{\gamma + 1} (2 - \beta^{\gamma + 1})}{(1-\beta)^2} = O(\beta^\gamma),
\tag{S2}
\label{eqn:egamma}
\end{equation*}
by Assumption 2, where $\sigma^2_{max} = \max_i \sigma^2_i$. \\

We now show that $|\rho(i,j \mid S)| = O(\| E_\gamma \|) = O(\beta^\gamma) $ where $S$ is a $\gamma$-local d-separator between $i$ and $j$. Let $A = \left\{i,j\right\} \cup S$ and $B = V \setminus A$. Consider the marginal precision matrix, $P := \{\Sigma_G(A,A)\}^{-1}$. Then, using the Schur complement, we can write this as 
\begin{equation*}
P = \Sigma_G^{-1}(A,A) - \Sigma_G^{-1}(A,B) \{ \Sigma_G^{-1}(B,B) \}^{-1} \Sigma_G^{-1}(B,A). 
\end{equation*}

Specifically, the partial correlation of $X_i$ and $X_j$ conditional on $S$ is given by $\frac{P_{1,2}}{(P_{1,1}P_{2,2})^{1/2}} = O(P_{1,2})$, by Assumption 2. \\

Recall from (S1) that $\Sigma_G = \Sigma_{H} + E_\gamma$. Let $F_\gamma$ be the matrix such that $\Sigma_G^{-1} = \Sigma_{H}^{-1} + F_\gamma$. Because $\Sigma_H$ only considers covariance induced by paths of length at most $\gamma$, we have that $\Sigma_{H}^{-1}(A,A)_{1,2} = 0$. \\

Thus,
\begin{align*}
| \{\Sigma_G(A,A) \}^{-1}_{1,2}| &= | \Sigma_G^{-1}(A,A)_{1,2} - \Sigma_G^{-1}(A,B) \{ \Sigma_G^{-1}(B,B) \}^{-1} \Sigma_G^{-1}(B,A)_{1,2} | \\
&= | \Sigma_{H}^{-1}(A,A)_{1,2} + F_{\gamma}(A,A)_{1,2} - \Sigma_G^{-1}(A,B) \{ \Sigma_G^{-1}(B,B) \}^{-1} \Sigma_G^{-1}(B,A)_{1,2} | \\
&= | F_{\gamma}(A,A)_{1,2} - \Sigma_G^{-1}(A,B) \{ \Sigma_G^{-1}(B,B) \}^{-1}  \Sigma_G^{-1}(B,A)_{1,2} | \\
&\le \| F_{\gamma}(A,A) - \Sigma_G^{-1}(A,B) \{ \Sigma_G^{-1}(B,B) \}^{-1}  \Sigma_G^{-1}(B,A) \|_{\infty} \\
&\le \| F_{\gamma}(A,A) - \Sigma_G^{-1}(A,B) \{ \Sigma_G^{-1}(B,B) \}^{-1} \Sigma_G^{-1}(B,A) \| .
\end{align*}

However, since $\Sigma_G^{-1}(A,B) \{ \Sigma_G^{-1}(B,B) \}^{-1} \Sigma_G^{-1}(B,A)$ is positive semi-definite, $| \{\Sigma_G(A,A) \}^{-1}_{1,2} | \le \| F_{\gamma}(A,A) \|$. We next show that $\|F_\gamma\| = O(\|E_\gamma\|) = O(\beta^\gamma)$. First, note that:
\begin{align*}
F_\gamma &= \Sigma_G^{-1} - \Sigma_{H}^{-1} \\
&= (\Sigma_{H} + E_\gamma)^{-1}  - \Sigma_{H}^{-1} \\ 
&= \Sigma_{H}^{-1} - \Sigma_{H}^{-1}(E_\gamma^{-1} + \Sigma_{H}^{-1})^{-1}\Sigma_{H}^{-1} -  \Sigma_{H}^{-1}  &&\text{(by Woodbury)} \\
&=  - \Sigma_{H}^{-1}(E_\gamma^{-1} + \Sigma_{H}^{-1})^{-1}\Sigma_{H}^{-1} 
\end{align*}

Then, taking spectral norms, and noting that $\Sigma_G = \Sigma_H + E_\gamma$:
\begin{align*}
\|F_\gamma\| &\le \|\Sigma_{H}^{-1}\| \|(E_\gamma^{-1} + \Sigma_{H}^{-1})^{-1}\| \|\Sigma_{H}^{-1}\| &&\text{(by sub-multiplicity)} \\
&= \|\Sigma_{H}^{-1}\|^2 \|E_\gamma - E_\gamma(\Sigma_{H} + E_\gamma)^{-1} E_\gamma \| &&\text{(by Woodbury)} \\
&= \|\Sigma_{H}^{-1}\|^2 \|E_\gamma(I - \Sigma_G^{-1}E_\gamma)\| \\
&\le \|\Sigma_{H}^{-1}\|^2 \|E_\gamma\| \|I - \Sigma_G^{-1}E_\gamma\| &&\text{(by sub-multiplicity)} \\
&\le \|\Sigma_{H}^{-1}\|^2 \|E_\gamma\| ( 1 + \|\Sigma_G^{-1}E_\gamma\|) &&\text{(by triangle inequality)} \\
&\le \|\Sigma_{H}^{-1}\|^2 \|E_\gamma\| ( 1 + \|\Sigma_G^{-1}\| \|E_\gamma\|)  &&\text{(by sub-multiplicity)} \\
&\le J \|E_\gamma\|^2, &&\text{(by boundedness of $\| \Sigma_{G}^{-1} \| \ge | \Sigma_{H}^{-1} \|$)} 
\end{align*}

for some constant $J$. Then, by walk-summability, $\|F_\gamma\| = O(|E_\gamma\|)$. Hence, $|\rho(i,j \mid S)| = O(\beta^\gamma)$ by (S1).

\end{proof}

By combining the result from either Lemma 1 or 2 with the $\lambda$-path-faithfulness assumption, we achieve oracle consistency for our algorithm given a threshold level $\alpha$ such that $\alpha = O(\lambda)$, $\alpha = \Omega(\beta^\gamma)$. 

Next, we consider the finite sample setting, and establish a concentration inequality for sample partial correlations, under sub-Gaussian distributions, using a result from \citet{ravikumar}. 

\begin{lemma}
Assume $X = (X_1, ..., X_p)$ is a zero-mean random vector with covariance matrix $\Sigma$ such that each $X_i / \Sigma_{ii}^{1/2}$ is sub-Gaussian with parameter $\sigma$. Assume $\|\Sigma\|_\infty$ and $\sigma$ are bounded. Then, the empirical partial correlation obtained from n samples satisfies, for some bounded $M > 0$:

\begin{equation*}
P\bigg( \max_{i \neq j, |S| \le \eta} |\hat{\rho}(i,j \mid S) - \rho(i,j \mid S)| > \epsilon \bigg) \le 4 \bigg(3 + \frac{3}{2}\eta + \frac{1}{2}\eta^2 \bigg) p^{\eta+2} \exp \bigg( -\frac{n\epsilon^2}{M} \bigg)
\end{equation*}

for all $\epsilon \le \max_i ({\Sigma_{ii}})8(1+4\sigma^2)$.
\end{lemma}

\begin{proof}

Using the recursive formula for partial correlation, for any $k \in S$
\begin{equation*}
\rho(i,j \mid S) = \cfrac{\rho(i,j \mid S \setminus k) - \rho(i, k \mid S \setminus k)  \rho(k, j \mid S \setminus k) } { (1 - \rho^2 (i, k \mid S \setminus k))^{1/2} (1 - \rho^2 (k, j \mid S \setminus k))^{1/2} }.
\end{equation*}

For example, with $S = \left\{ k \right\}$, we can simplify this to:
\begin{equation*}
\rho(i,j \mid S) = \cfrac{\rho(i,j) - \rho(i, k)  \rho(k, j) } { (1 - \rho^2 (i, k))^{1/2} (1 - \rho^2 (k, j))^{1/2} },
\end{equation*}
where $\rho(i,j) = \Sigma_{ij} / (\Sigma_{ii}\Sigma_{jj})^{1/2}$.

Rewriting in terms of elements of $\Sigma$, we then decompose the empirical partial correlation deviance from the true partial correlation into the deviances of covariance terms. Here, the event of the empirical partial correlation being within $\epsilon$ distance of the true partial correlation is contained in the union of the empirical covariance terms being within $C \epsilon$ distance of the true covariance terms for a sufficiently large $C > 0$:
\begin{align*}
\bigg[  |\hat{\rho}(i,j \mid S) - \rho(i,j \mid S)| > \epsilon \bigg] &\subset \bigg[  |\hat{\Sigma}_{ij} - \Sigma_{ij}| > C\epsilon \bigg] \bigcup \bigg[  |\hat{\Sigma}_{ii} - \Sigma_{ii}| > C\epsilon \bigg] \bigcup \bigg[  |\hat{\Sigma}_{jj} - \Sigma_{jj}| > C\epsilon \bigg] \\
&\bigcup_{k \in S} \bigg[  |\hat{\Sigma}_{ik} - \Sigma_{ik}| > C\epsilon \bigg] \\
&\bigcup_{k \in S} \bigg[  |\hat{\Sigma}_{jk} - \Sigma_{jk}| > C\epsilon \bigg] \\
&\bigcup_{k \le k' \in S} \bigg[  |\hat{\Sigma}_{kk'} - \Sigma_{kk'}| > C\epsilon \bigg],
\end{align*}

The number of events on the right-hand side is $3 + |S| + |S| + |S|^2 - $ $|S| \choose 2$. For $|S| \le \eta$, the number of events is then bounded by $3 + \frac{3}{2}\eta + \frac{1}{2}\eta^2$. Then, by applying Lemma 1 in \citet{ravikumar}, we have that, for any $i,j$:

\begin{equation*}
Pr\bigg(|\hat{\rho}(i,j \mid S) - \rho(i,j \mid S)| > \epsilon \bigg) \le 4 \bigg(3 + \frac{3}{2}\eta + \frac{1}{2}\eta^2 \bigg) \exp \bigg( -\frac{n\epsilon^2}{K} \bigg),
\end{equation*}

for some $K > 0$, bounded when $\|\Sigma\|_\infty$ and $\sigma$ are bounded. From here, the result follows.

\end{proof} 

Combining the results established in Lemmas 1-3, we now prove the consistency of our algorithm in the finite sample setting. 

Let $\alpha$ denote the threshold where if $\hat{\rho}_G(i,j \mid S) < \alpha$, the edge $(i,j)$ is deleted. Let $G_S$ denote the true undirected skeleton of $G$, and let $S_{\eta, \gamma}$ denote the set of $\gamma$-local d-separators or size at most $\eta$. \\

For any $(i,j) \not \in G_S$, define the false positive event as 
\begin{equation*}
F_1(i,j) = \bigg[ \min\limits_{ \substack{S \in S_{\eta, \gamma} } } |\hat{\rho}_G(i,j \mid S)| > \alpha \bigg].
\end{equation*}

Define 
\begin{equation*}
\theta_{max} = \max\limits_{(i,j) \not \in G_S} \min\limits_{ S \in S_{\eta, \gamma} } |\rho_G(i,j \mid S)| 
\end{equation*}
and
\begin{equation*}
\hat{\theta}_{max} = \max\limits_{(i,j) \not \in G_S} \min\limits_{ S \in S_{\eta, \gamma} } |\hat{\rho}_G(i,j \mid S)|.
\end{equation*}

Consider
\begin{align*}
Pr \bigg\{ \bigcup_{(i,j) \not \in G_S} F_1(i,j) \bigg\}
&= Pr( \hat{\theta}_{max} > \alpha) \\ 
&= Pr( |\hat{\theta}_{max} - \theta_{max}| > | \alpha - \theta_{max} | ) \\ 
&= O\bigg[p^{\eta+2} \exp \bigg\{- \frac{n(\alpha - \theta_{max})^2}{M} \bigg \} \bigg] &&\text{(by Lemma 3)} 
\end{align*}
where $\theta_{max} =  O(\beta^\gamma)$ by Lemma 1 and 2.

For any true edge $(i,j) \in G_S$, define the false negative event as
\begin{equation*}
F_2(i,j) = \bigg[ \min\limits_{\substack{ S \subset V \setminus \left\{i,j\right\}, |S| \le \eta } } |\hat{\rho}_G(i,j \mid S)| < \alpha \bigg].
\end{equation*}

Define 
\begin{equation*}
\theta_{min} = \min_{(i,j) \in G_S} \min\limits_{ S \subset V \setminus \left\{i,j\right\}, |S| \le \eta} |{\rho}_G(i,j \mid S)| 
\end{equation*}
and
\begin{equation*}
\hat{\theta}_{min} = \min_{(i,j) \in G_S} \min\limits_{ S \subset V \setminus \left\{i,j\right\}, |S| \le \eta} |\hat{\rho}_G(i,j \mid S)|.
\end{equation*}

Consider
\begin{align*}
Pr \bigg\{ \bigcup_{(i,j) \in G_S} F_2(i,j) \bigg\}
&= Pr( \hat{\theta}_{min} < \alpha) \\ 
&= Pr( | \theta_{min}  - \hat{\theta}_{min} | > | \theta_{min} - \alpha |) \\ 
&= O\bigg[p^{\eta+2} \exp \bigg\{ - \frac{n(\alpha - \theta_{min})^2}{K} \bigg\} \bigg] &&\text{(by Lemma 3)} 
\end{align*}
where $\theta_{min} =  \Omega(\lambda)$ by restricted path-faithfulness assumption. \\

Under our assumptions, we have that $n = \Omega \{ (\log p)^{1/1 - 2c} \}$, and $\lambda = \Omega(n^{-c})$ with $c \in (0, 1/2)$. Rewriting in terms of $\lambda$, we have $n = \Omega ( \frac{\log p}{\lambda^2} )$. Then, by selecting $\alpha$ such that $\alpha = O(\lambda)$, $\alpha = \Omega(\beta^\gamma)$, we have $Pr \{ \bigcup_{(i,j) \not \in G_S} F_1(i,j) \} = o(1)$ and $Pr \{ \bigcup_{(i,j) \in G_S} F_2(i,j) \} = o(1)$. This completes the proof of Theorem~\ref{thm::consistency}. 

\section*{Appendix B. Additional Simulation Results}
\label{app:sims}

In this section, we present some simulation results comparing our algorithms to the PC-Algorithm in estimating dense graphs. The simulation setup is otherwise identical to that described in the main paper. We consider a low-dimensional setting, with $p = 100$ and $n = 200$, as well as a high-dimensional setting, with $p = 200$ and $n = 100$. For Erd\H{o}s-R{\'e}nyi graphs, we use a constant edge probability of 0.05, corresponding to an expected degree of 5 for the low-dimensional graph and 10 for the high-dimensional graph. For power law graphs, we use an expected degree of 6 in both graphs. 

\begin{figure}[h!]
\centering
\includegraphics[width=0.75\textwidth]{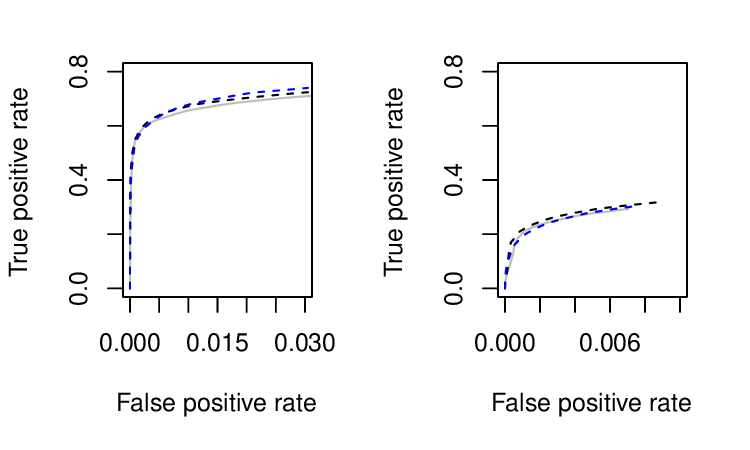}
\caption{Average true vs. false positive rates for PC-Algorithm (grey solid line), rPC-approx (black dashed line), and rPC-full (blue dashed line) estimating Erd\H{o}s-R{\'e}nyi graphs. Left: $p=100, n = 200$, average degree 5; right: $p = 200, n = 100$, average degree 10.}
\label{fig:ERdense}
\end{figure}  

\begin{figure}[h!]
\centering
\includegraphics[width=0.75\textwidth]{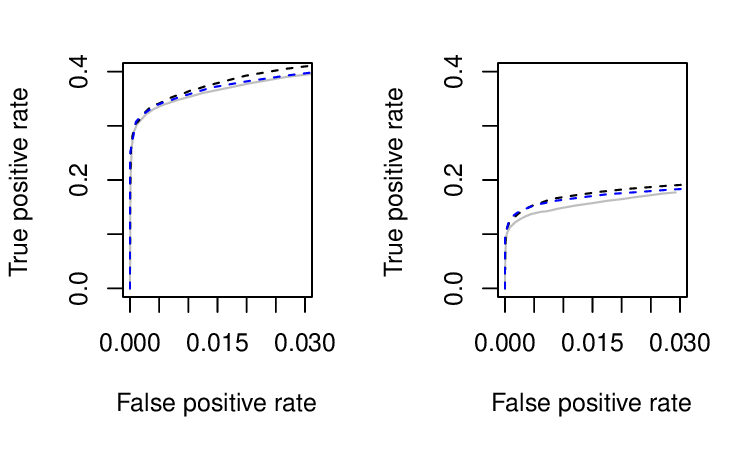}
\caption{Average true vs. false positive rates for PC-Algorithm (grey solid line), rPC-approx (black dashed line), and rPC-full (blue dashed line) estimating power law graphs with average degree 6. Left: $p=100, n = 200$; right: $p = 200, n = 100$.}
\label{fig:PLdense}
\end{figure}  

While estimation quality is worse in the dense setting for all methods, we observe in Figure~\ref{fig:ERdense} and Figure~\ref{fig:PLdense} similar trends as in the main paper. Our methods perform as well as the PC-Algorithm in estimating Erd\H{o}s-R{\'e}nyi graphs, and shows some improvement for power law graphs. Both versions of rPC give similar results.  

We also include more simulation results comparing path faithfulness to the restricted strong faithfulness assumption. As in Section~\ref{sec:faithfullness}, 1000 random DAGs were generated from Erd\H{o}s-R{\'e}nyi and power law families, with edge weights drawn independently from a $\text{Uniform}(-1, 1)$ distribution. Here, we consider graphs of size $p = 10$ and $p = 20$, with faithfulness condition parameters $\lambda = 0.001$ and $\eta = 2$. The results are shown in Table~\ref{tab:faithsim10} and Table~\ref{tab:faithsim30}. 

\begin{table}
\caption{\label{tab:faithsim10} Empirical probabilities of random DAGs of size $p = 10$ satisfying faithfulness conditions; RSF refers to restricted strong faithfulness of the PC-Algorithm, and PF refers to path faithfulness of reduced PC (rPC).}
\centering
\begin{tabular}{cccc}
Graph family & Expected degree & $Pr(\text{RSF})$ & $Pr(\text{PF})$ \\
\hline
Erd\H{o}s-R{\'e}nyi           & 2               & 0.94     & 0.98    \\
Erd\H{o}s-R{\'e}nyi           & 5               & 0.10      & 0.61     \\
Power law           & 2               & 0.95     & 0.97   \\
Power law           & 6               & 0.08      & 0.60   
\end{tabular}
\end{table}

\begin{table}
\caption{\label{tab:faithsim30} Empirical probabilities of random DAGs of size $p = 30$ satisfying faithfulness conditions; RSF refers to restricted strong faithfulness of the PC-Algorithm, and PF refers to path faithfulness of reduced PC (rPC).}
\centering
\begin{tabular}{cccc}
Graph family & Expected degree & $Pr(\text{RSF})$ & $Pr(\text{PF})$ \\
\hline
Erd\H{o}s-R{\'e}nyi           & 2               & 0.66     & 0.86    \\
Erd\H{o}s-R{\'e}nyi           & 5               & 0      & 0.01     \\
Power law           & 2               & 0.04     & 0.51    \\
Power law           & 6               & 0      & 0.01   
\end{tabular}
\end{table}

Although it is harder for both conditions to be met as graph size and density increase, we see that path faithfulness is still more likely to be satisfied than restricted strong faithfulness, particularly with power law graphs. 

Finally, we consider the plausibility of Assumption~\ref{assum::bddtrek} when the SEM coefficients are not bounded by 1 in absolute value. Suppose that the data matrix $X$ has been normalized by column-wise standard deviations. We show that this transformation preserves the original network structure, and leads to most edge weights being bounded by 1 in absolute value. Consider an edge $j \rightarrow k$ and let $W_j := \left\{ X_i : i \in pa(k) \setminus \left\{ j \right\} \right\}$. Then, taking the conditional covariance:
\begin{align*}
Cov(X_j, X_k | W_j ) &= Cov \left(X_j, \rho_{jk} X_j + \sum_{i \in pa(k)\setminus j} \rho_{ik} X_i + \epsilon_k \bigg| W_j \right) \\
&= \rho_{jk}Var(X_j | W_j ).
\end{align*}
We can therefore write:
\begin{equation*}
\rho_{jk} = \cfrac{Cov(X_j, X_k | W_j)}{Var(X_j | W_j)}.
\end{equation*}
Now, let $\tilde{X}_k = {X}_k/sd({X}_k)$ for all $k$. Consider the edge weights $\tilde{\rho}_{jk}$ corresponding to the SEM for this transformed data. We have:
\begin{equation*}
\tilde{\rho}_{jk} = \cfrac{sd(X_j)}{sd(X_k)} \cfrac{Cov(X_j, X_k | W_j)}{Var(X_j | W_j)} = \cfrac{sd(X_j)}{sd(X_k)} \rho_{jk} 
\end{equation*}
Clearly, $\tilde{\rho}_{jk} = 0$ if and only if $\rho_{jk} = 0$. Therefore, we recover the same network by applying our algorithm to the transformed data. 
Furthermore, 
\begin{align*}
|\tilde{\rho}_{jk}| &= \sqrt{\frac{Var(X_j)}{Var(X_k)} \rho^2_{jk}} \\
&= \sqrt{\frac{Var(X_j) \rho^2_{jk}}{Var(X_j) \rho^2_{jk} + \sigma^2_k + \sum_{i \in pa(k) \setminus \left\{j\right\} } \rho^2_{ik} Var(X_i) + 2 \sum_{i, i' \in pa(k)} \rho_{ik}\rho_{i'k} Cov(X_i, X_i') }}
\end{align*}

Thus, $|\tilde{\rho}_{jk}| > 1$ only if
\begin{equation*}
\sigma^2_k + \sum_{i \in pa(k) \setminus \left\{j\right\} } \rho^2_{ik} Var(X_i)  + 2 \sum_{i, i' \in pa(k)} \rho_{ik}\rho_{i'k} Cov(X_i, X_i') < 0
\end{equation*}
Rewriting the variance and covariance terms as trek sums, we obtain:
\begin{equation*}
\sigma^2_k + \sum_{i \in pa(k) \setminus \left\{j\right\} } \rho^2_{ik} \sum_{ \substack{\pi: u \leftrightarrow i \\ u \in V \setminus \left\{i\right\} } } \sigma^2_w \prod_{\rho_e \in \pi} \rho^2_e  + 2 \sum_{i, i' \in pa(k)} \rho_{ik}\rho_{i'k} \sum_{ \pi: i \leftrightarrow i' } \sigma_w \prod_{\rho_e \in \pi} \rho_e  < 0
\end{equation*}
where $w$ denotes the source or common node in the trek. 

Since the first two terms in this expression will always be positive, $|\tilde{\rho}_{jk}| > 1$ only if the covariance sum is both negative and greater in magnitude than the sum of the first two terms. A term in the covariance sum will be negative if there is an odd number of negative terms in the product, and the entire sum will be negative only if there are a sufficient number of these to make the whole sum negative. Even then, the sum over squared trek products would need to be smaller than the covariance term sum in absolute value.
Moreover, even if some $|\tilde{\rho}_{jk}| > 1$, the product over the entire trek may still be less than 1 if most of the other coefficients along a trek are small. Thus, we may still expect the product over the trek to be small. 

Next, we provide simulation results showing that standardized SEM coefficients in graph configurations we consider rarely exceed 1 in absolute value. Specifically, we consider Erd\H{o}s-R{\'e}nyi and power law graphs, with $p \in \left\{ 200, 500 \right\}$ having average degree 2. We consider edge weights from both a $\mbox{Unif}(-10, 10)$ and $N(0,3^2)$ distribution. For each graph, we estimate the standardized coefficients using the \texttt{fitDag} function in the \texttt{ggm} library, and compute the proportion of edge weights that are greater than 1 in absolute value. We repeat this process over 1,000 iterations. The results are shown in Table~\ref{tab:coefsim}. This experiment shows that, with high probability, the covariance induced over long treks between any two nodes is indeed small.

\begin{table}
\caption{\label{tab:coefsim} Empirical average probabilities of standardized coefficients exceeding 1 in absolute value.}
\centering
\begin{tabular}{cccc}
Graph family & $p$ & Distribution & $Pr(|\rho| > 1)$ \\
\hline
Erd\H{o}s-R{\'e}nyi           & 200               & Uniform(-10, 10)  & 0.0036    \\
           & 200               & Normal(0, $3^2$)     & 0.0017   \\
           & 500               & Uniform(-10, 10)     & 0.0036    \\
           & 500               & Normal(0, $3^2$)     & 0.0012    \\
Power law           & 200               & Uniform(-10, 10)     & 0.0045    \\
           & 200               & Normal(0, $3^2$)     & 0.0045    \\
           & 500               & Uniform(-10, 10)     & 0.0034    \\
           & 500               & Normal(0, $3^2$)     & 0.0037    \\
\end{tabular}
\end{table}

\clearpage

\bibliographystyle{plainnat}
\bibliography{manuscriptrefs}

\end{document}